\documentclass[letterpaper,10pt,conference]{ieeeconf} 
\IEEEoverridecommandlockouts
\pdfoutput=1

\makeatletter

\let\proof\@undefined
\let\endproof\@undefined
\makeatother

\usepackage{amsmath,amssymb,amsthm}
\usepackage{graphicx,url}
\usepackage{epsfig,color}
\usepackage{xspace}
\usepackage{floatrow}
\usepackage{subfigure}
\usepackage{multirow}
\usepackage{epstopdf}

\newtheorem{theorem}{Theorem}[section]
\newtheorem{proposition}[theorem]{Proposition}

\newtheorem{lemma}[theorem]{Lemma}

\theoremstyle{definition}
\newtheorem{definition}[theorem]{Definition}

\newtheorem{example}[theorem]{Example}
\theoremstyle{remark}
\newtheorem{remark}[theorem]{Remark}

\graphicspath{{./fig/}}

\newcommand{\prb}[1]{\ensuremath{\mathbb{P}\left[ #1 \right] }} 
\newcommand{\tl}{t_{\mathrm{lag}}}
\newcommand{\td}{t_{\mathrm{det}}}
\newcommand{\tc}{t_{\mathrm{conf}}}

\usepackage{todonotes} 

\newcommand\oprocendsymbol{\hbox{$\bullet$}}
\newcommand\oprocend{\relax\ifmmode\else\unskip\hfill\fi\oprocendsymbol}

\urldef{\smith}\url{stephen.smith@uwaterloo.ca}
\urldef{\ahmad}\url{abasghar@uwaterloo.ca}

\title{\Large \bf Robot Monitoring for the Detection and Confirmation of Stochastic Events}  
\author{Ahmad Bilal Asghar \qquad Stephen L. Smith \thanks{This research is partially
    supported by the Natural Sciences and Engineering Research Council
    of Canada (NSERC). }
  \thanks{The authors are with the Department of Electrical and
    Computer Engineering, University of Waterloo, Waterloo ON, N2L 3G1
    Canada  (\ahmad; \smith) }}    
\date{}      

\begin{document}           

\maketitle                 

\begin{abstract}
In this paper we consider a robot patrolling problem in which events arrive randomly over time at the vertices of a graph.  When an event arrives it remains active for a random amount of time.  If that time active exceeds a certain threshold, then we say that the event is a true event; otherwise it is a false event.  The robot(s) can traverse the graph to detect newly arrived events, and can revisit these events in order to classify them as true or false.  The goal is to plan robot paths that maximize the number of events that are correctly classified, with the constraint that there are no false positives.  We show that the offline version of this problem is NP-hard.  We then consider a simple patrolling policy based on the traveling salesman tour, and characterize the probability of correctly classifying an event.  We investigate the problem when multiple robots follow the same path, and we derive the optimal (and not necessarily uniform) spacing between robots on the path.
\end{abstract}

\section{Introduction}  

Consider the following motivating example.  An autonomous robotic vehicle traverses a parking lot, issuing tickets to vehicles that have overstayed the allowed amount of parking time $T$. The robot goes from spot to spot, recording the time and license plate numbers of the parked vehicles. If a vehicle has been present for more than $T$ time units after it was first spotted by the robot, then it gets a ticket. However, there is a possibility that some vehicles overstay their allowed parking time but are not ticketed. Our goal is to define a monitoring policy for the robot which minimizes the number of un-ticketed, overstaying vehicles. 

Formally, the \emph{event detection and confirmation problem} considered in this paper is as follows.  We are given a weighted graph.  Events arrive at the vertices of the graph, and a robot, or group of robots, can patrol the graph by traversing its edges.  Once an event arrives at a particular vertex it remains active for a randomly distributed amount of time.  If the event remains active for more than a given threshold time $T > 0$, then we say it is a \emph{true event}, otherwise it is a \emph{false event}.  For a robot to verify that an event is true, it must first detect the event by visiting the vertex, and then must revisit the vertex at least $T$ time units later to confirm the event.  Thus, the goal is for the robots to maximize the expected number of true events that are successfully confirmed.  This is a classification problem in which false positives are not permitted:  Each event is initially classified as false, and it can be classified as true only if it is confirmed.

\emph{Related work}: While the proposed event detection and confirmation problem has not, to our knowledge, been directly studied there are several closely related problems.  In the patrolling problem~\cite{chevaleyre2004theoretical, Agmon2008Multi, baseggio2010distributed}, the goal is to monitor an environment or boundary using one or more robots/sensors.  The performance criteria is to minimize the maximum time between visits to any region in the environment.  In~\cite{chevaleyre2004theoretical}, the problem is considered for multiple robots, and it is shown that good patrolling performance can be achieved by computing a single traveling salesman tour (TSP)~\cite{BK-JV:07}, and then equally distributing the robots along this tour.

In~\cite{Pasqualetti2012Cooperative}, the authors look at cooperative patrolling problems and give approximation algorithms for certain classes of discrete environments. In~\cite{Alamdari2014Persistent} the patrolling is extended to environments in which each region has a different importance level, and the goal is to minimize the time between visits to a region, weighted by that regions importance.  The work in this paper can be thought of as a natural extension of patrolling in which an action must be taken if an event is detected during the patrol (that action being confirmation).

Our problem is also related to the TSP with time windows~\cite{savelsbergh1985local, GL:09}, where the input is a graph along with a time window assigned to each vertex.  The goal is to find the shortest tour that visits each vertex exactly once, and within its time window.  We show that the event confirmation aspect of our problem is closely related to TSP with time windows, since each event must be confirmed at least $T$ time units after detection, but before the event expires.

Another closely related problem is the pickup and delivery problem~\cite{ropke2006adaptive}, where one seeks to pickup a set of customers at their desired origin locations and drop them off at their desired destination locations, all within their specified time windows.  Our problem can be thought of as a variation in which the pickup time (i.e., the event arrival time) is unknown to the robot, the pickup and destination locations coincide, and the dropoff time window depends on the time that the pickup occurred (i.e., the event was detected).

The stochastic aspect of the problem bears a close resemblance to dynamic vehicle routing (DVR)~\cite{Bullo2011Dynamic}, where spatially distributed customers arrive stochastically over time, and the goal is to minimize the expected time between a customers arrival and the time it is visited by a vehicle.  The most closely related work in this area is~\cite{MP-NB-EF-VI:08}, in which the customers exit the system if they are not visited within a time window.  However, DVR differs from the proposed work in three regards: i) the environment is a Euclidean space rather than a graph, ii) the customer is known to the vehicles upon arrival, and iii) a second confirmation visit is not required.

\emph{Contributions:} In this paper we introduce the event detection and confirmation problem.  There are three main contributions.  First, we characterize the complexity of the problem by relating its offline counterpart to the TSP with time windows.  Second, we propose a simple periodic visit strategy based on the TSP and analyze the probability of confirming a true event.  Third, we give some insight into the multi-robot problem, and show that unlike traditional patrolling~\cite{chevaleyre2004theoretical} when robots are all placed on the same path, it is not always optimal for them to be equally spaced.

\emph{Organization:}  In Section~\ref{sec:problem_statement} we formally define the problem and in Section~\ref{sec:offline} we give its corresponding offline version and show it is NP-hard. In Section~\ref{sec:single_vertex} we consider the single robot problem and in Section~\ref{sec:multiple_robots} we give some initial analysis into the multirobot setting. Finally, in Section~\ref{sec:conclusions} we give some future directions for research.

\subsection{Preliminaries}
\label{sec:preliminaries}
We require a few basic properties of the Poisson and exponential distributions~\cite{GG-DR:01}.  The \emph{exponential distribution} with parameter $\mu$ is a continuous distribution with a probability density function of
\[
f(x) = \begin{cases}
\mu e^{-\mu x} & \text{if $x \geq 0$}, \\
0 & \text{if $x < 0$}.
\end{cases}
\]
The cumulative distribution function of an exponential random variable is $F(x) = 1 - e^{-\mu x}$ if $x\geq 0$ and $F(x) = 0$ otherwise.

A \emph{Poisson process} with parameter $\lambda$ is a stochastic counting process such that the time between successive events is exponentially distributed.  The expected number of event arrivals in a time interval $[t_1,t_2]$ is $\lambda (t_2-t_1)$.  The Poisson process also satisfies the property of \emph{stationary increments} where the number of arrivals in an interval of time is independent of the number of arrivals prior to that interval.  

The number of arrivals in the time interval $(0,t]$ (and thus any interval of length $t$) is denoted $N(t)$ and distributed according to the Poisson distribution:
\[
\prb{N(t) = k} = \frac{(\lambda t)^ke^{-\lambda t}}{k!}.
\]

The following result will be useful in our analysis.

\begin{lemma}[Poisson Arrival Time Distribution,~\cite{ross1996stochastic}]
\label{lem:arrival_dist}
  Given that $k$ events arrived in the time interval $(a,b]$, the times $t_1,t_2,\ldots,t_k$ of these arrivals, considered as unordered random variables, are independent and uniformly distributed on $(a,b]$. 
\end{lemma}

A consequence of this result is that if we know an event arrived in an interval of time, then its arrival time is uniformly distributed over that time interval.


\section{Problem Statement and Hardness}
\label{sec:problem_statement}

In this section we introduce the event detection and confirmation problem and we characterize its hardness.   

\subsection{Problem Statement}

The Event Detection and Confirmation problem is defined on an undirected weighted graph $G=(V,E,w)$, where $V$ is the vertex set, $E$ is the set of edges and $w:E\rightarrow\mathbb{R}$ represents edge weights. The vertices depict the locations to be monitored by $m\geq1$ robots. We take the metric closure~\cite{Alamdari2014Persistent} of $G$ in order to obtain a complete graph, in which the length of each edge is equal to the shortest path distance in the original graph.  For simplicity we will refer to this complete graph as $G=(V,E,w)$.

Events arrive at each vertex $v\in V$ according to a Poisson random process~\cite{GG-DR:01} with a parameter $\lambda_v$. Similarly, we assume that the activity period of an event at a vertex is exponentially distributed with parameter $\mu_v$.\footnote{In queueing theory, the Poisson distribution and exponential distribution are often used to model customer arrival rates and customer service times, respectively~\cite{Kleinrock1975Queueing}.  Most of the analysis in this paper holds for more general distributions:  the ability to obtain closed-form expressions, however, leverages these specific distributions.} The events are distinct and they can be identified by the robots. Moreover, only one event can be active at a vertex at a time, and the next event at that vertex can only arrive after the previous one has gone inactive. The arrivals and active times of events at different vertices are independent. 

There is also a critical time $T$ as input to the problem. We call an event a \emph{true event} if it remains active for an amount of time that is greater than or equal to $T$. The robots, while on their patrolling path, perform two tasks: \emph{detection} and \emph{confirmation}. The detection of an event is discovering it for the first time at a vertex, and the confirmation is observing an event at a vertex after it has been active for at least time $T$. The robots can classify an event as true if and only if they confirm that event. Notice that if a true event becomes inactive before being confirmed, it cannot be classified by the robot as a true event.

When a robot reaches a vertex in its tour, it faces one of the following scenarios:
\begin{enumerate}
\item The vertex is empty (no event at the vertex): then, the robot can delete the event from its database which was recorded to be at that vertex (if any);
\item There is a new event at the vertex: then the robot stores it against that vertex with the current time stamp;
\item There is an event at the vertex which was detected some previous check at that vertex:  In this case the robot looks up the time stamp of that event and compares it with current time to see whether it is a true event or not.
\end{enumerate}

\textbf{Event detection and confirmation problem:} Find patrolling paths for the robots to minimize the probability of incorrectly classified events. The problem does not allow the robots to classify a false event as true, so the optimization task can be stated as maximizing the probability of correctly classified true events.
\begin{figure}
  \centering
 \includegraphics[width=0.9\linewidth]{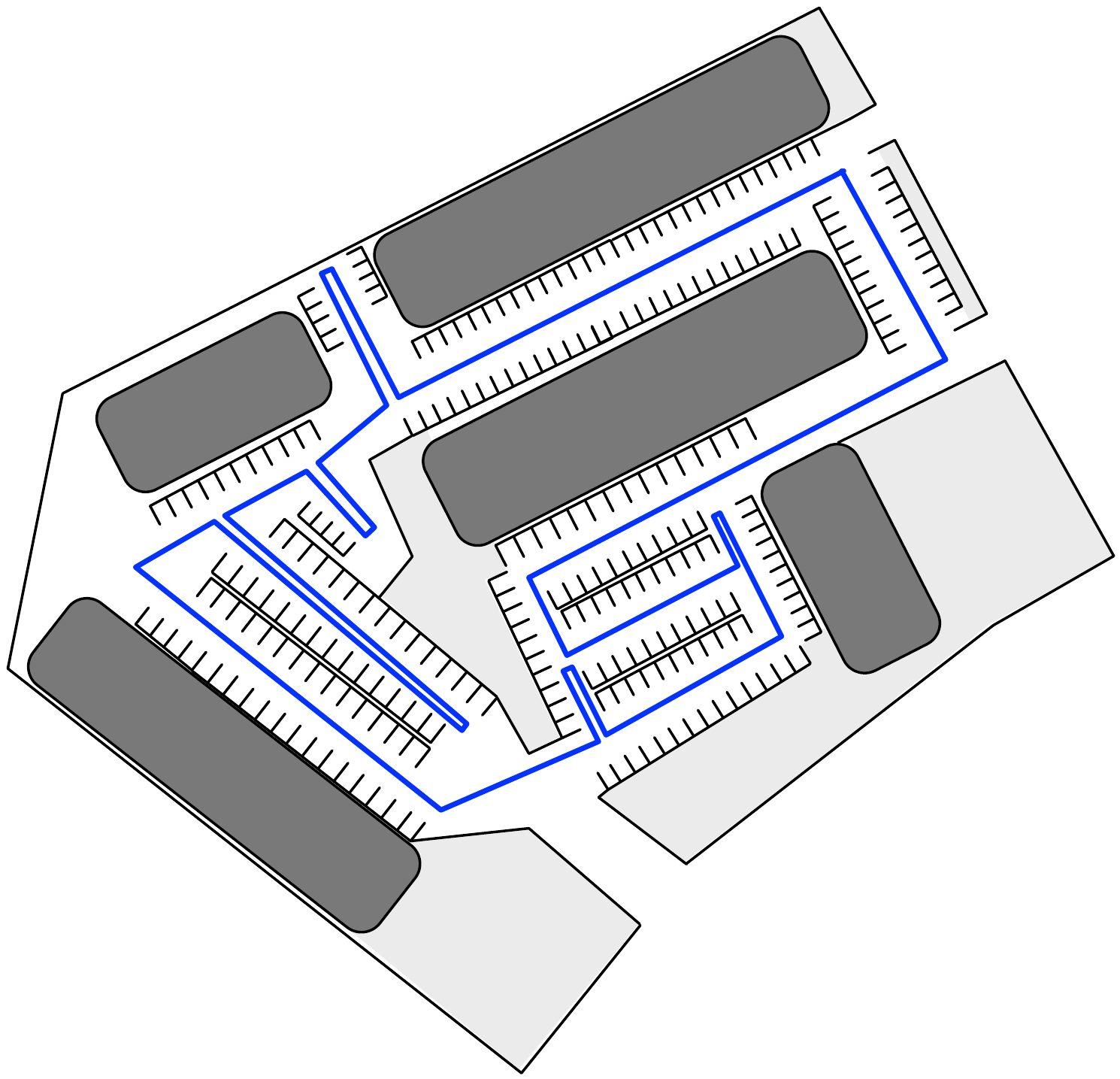}
  \caption{A parking lot environment and a possible patrolling tour based on the TSP.  Here we assume that when moving along an aisle in the parking lot, the robot can view the parking spots on either side of the aisle.}
  \label{fig:parking_example}
\end{figure}

\begin{example}[Parking Enforcement]
The problem is analogous to the problem of finding routes in a parking lot where cars are arriving and departing according to some random process and we have to ticket the cars staying more than the allowed parking time. The vertices of the graph represent individual parking spaces and the arrival and activity time of events at a vertex is equivalent to arrival and staying time of vehicles at a parking spot. The cars are identifiable by their license plates and there can be only one car at a spot. Moreover, the robot can ticket a car only if the time between its detection and ticketing is more than the allowed parking time. An example parking lot environment and a possible patrolling route are shown in Figure~\ref{fig:parking_example}. \oprocend
\end{example}

\subsection{Off-line Version and Hardness}
\label{sec:offline}

To better understand the complexity of the Event Detection and Confirmation Problem, let us consider its off-line version.  In the off-line problem all the required data is available beforehand. So the arrival times and activity periods of events at each vertex of the graph are known before devising a patrolling path. The input to the off-line problem are: a weighted graph $G=(V,E,w)$, the critical time $T$, and $M$ number of events where each event is given by $E_i=(v,t_s,t_f), v\in V, i\in \{1,2,..,M\}$. The vertex of the event is represented by $v$, and $[t_s,t_f]$ gives the interval during which the event remains active. More than one event can arrive at a vertex, but their active intervals need to be disjoint. Given this data, the events with active times less than the minimum staying time $T$ can be ignored. For the other(true) events, the monitoring robot needs to visit their vertex twice: once after their arrival to detect them and then before $t_f$ but at least time $T$ after the first visit to confirm them. This means that we can assign two time windows at each vertex and the robot has to visit each vertex during those time windows. We name the windows as detection window and confirmation window, given by $[t_s,t_f-T]$ and $[t_{det}+T,t_f]$ respectively,  where $t_{det}\in[t_s,t_f-T]$ represents the actual visit time of vertex $v$ by the robot during the first window. Note that the confirmation window is dynamic in the sense that its length depends on the visit time during the detection window.

\begin{proposition}[Hardness of Offline Problem]
The problem of finding a feasible tour for the off-line version of Event Detection and Confirmation Problem(EDC) is NP-Complete.
\end{proposition}

\begin{proof}
The following reduction from the decision version of {\sc traveling salesman problem with time windows(TSPTW)} to the off-line EDC shows that the off-line EDC problem is at least as hard as the problem of finding a feasible solution in TSPTW, which is known to be NP-Complete~\cite{savelsbergh1985local}.

The TSPTW takes as an instance a weighted graph $G=(V',E',w')$, and a time window for each vertex $i\in V'$, denoted by $[e_i, l_i]$. The decision version of the problem is to find whether a tour exists which visits each vertex $i\in V'$ of the graph within the time window $[e_i, l_i]$ associated with that vertex. 

Given an instance of TSPTW, we generate the following instance of off-line EDC. 
\begin{equation*}
\begin{aligned}
(V,E,w)&=(V',E',w')\\
T&=\sum_{i,j}w_{ij} + \max_{i\in V'}\{l_i\} - \min_{i\in V'}\{e_i\}\\
E_i=(v,t_s,t_f)&=(i,e_i,l_i+T) \quad \forall i\in V'
\end{aligned}
\end{equation*}

The detection window for the off-line EDC at vertex $v$ will therefore be
\begin{equation*}
[t_s,t_f-T]=[e_i,l_i].
\end{equation*}

If a patrolling path for off-line EDC exists, that means that the robot can visit each vertex in its detection window, which is same as the time window for TSPTW.

Similalry, if a feasible solution to the TSPTW exists, then the robot can visit all the vertices of EDC during their detection windows. Let that path be called `detection path'. Then it can also visit all the vertices in their confirmation windows by following the detection path with a time lag of $T$. This is possible because the robot will cover all the vertices in $t_1 \leq\max_{i}\{l_i\} - \min_{i}\{e_i\}$, and it can go to the first vertex on the detection path from the last in $t_2 \leq\sum_{i,j}w_{ij}$. So, once it reaches the starting vertex of the detection path, it can wait for $T-t_1-t_2$ time and then follow the same path to visit all the vertices in their confirmation windows.

Therefore, the decision version of TSPTW is true if and only if off-line EDC is true.  Moreover, given any certificate of off-line EDC, it can be easily checked in polynomial time whether it visits all the vertices in their corresponding detection and confirmation windows. Hence, off-line EDC is NP-Complete.  
\end{proof}

The above result shows us that it is computationally intractable to even determine a feasible patrolling path given all the problem data beforehand.  This implies that the optimization version of the offline problem is NP-hard.  Thus, we do not expect there to exist a tractable algorithm for optimally solving the online problem, in which events arrive over time.    In the following section we look at the probability of confirmation at a single vertex in the graph.  Using this we can analyze the performance of a policy based on a traveling salesman problem (TSP) tour.

\section{Analysis for a Single Vertex}
\label{sec:single_vertex}
Let us consider a deterministic patrolling policy which periodically visits each vertex, and let the visit period for a vertex $v$ be $\tau$. Inter-arrival and staying times of events at vertex $v$ are distributed exponentially with the parameters $\lambda$ and $\mu$ respectively (we drop the subscript $v$ in this section for simplicity of notation). Since the arrival and departure process at a vertex are independent of the states of other vertices, we can focus the analysis on a single vertex. 

\subsection{Confirmation Avoiding Interval for True Events}

Our end goal is to maximize the probability of correctly classified true events. There is a chance that a true event becomes inactive without being confirmed, because the robot does not know the exact arrival time of the event and it can only visit the vertex of the event periodically. In the following we characterize the time interval on which a true event can become inactive and avoid confirmation.

\begin{figure}  
\label{fig_stay_time}
\centering  
\includegraphics[width=\linewidth]{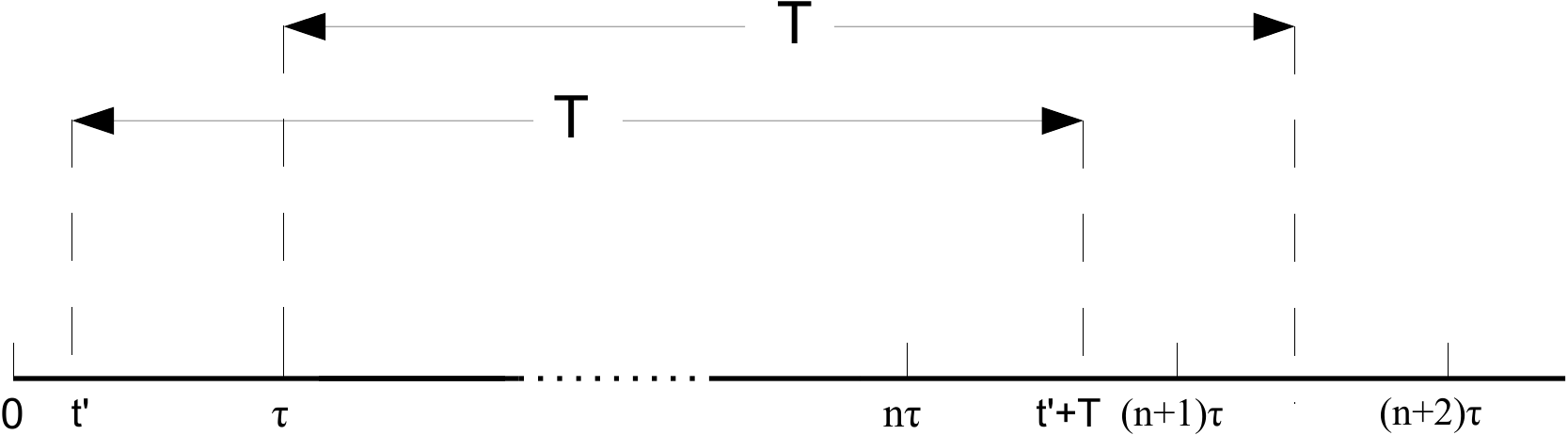}  
\caption{The events remaining active until $t'+T$ are true, but they cannot be confirmed unless they stay until $(n+2)\tau$.}  
\end{figure}  

\begin{proposition}[False Negatives]
\label{prop:false_negative_interval}
Suppose an event arrives at a vertex between two consecutive visits made by the robot at $0$ and $\tau$, and the arrival time of the event is given by $t'\in(0,\tau]$.  Then, if that event becomes inactive in the time interval
\begin{equation}
\label{eq:conf_interval}
\big(t'+T,(n+2)\tau\big), \; \text{where} 
\end{equation}
\begin{equation}
\label{eq:n_def}
 n=\begin{cases}
    \frac{T}{\tau}-1, & \text{if $T$ is a multiple of $\tau$}\\
    \big\lfloor\frac{T}{\tau}\big\rfloor, & \text{otherwise},
  \end{cases}
\end{equation}
it will be a true event which can not be confirmed.
\end{proposition} 

\begin{proof}
The starting point of the interval $(t'+T,(n+2)\tau)$  is trivial since
the event will become true after $t=t'+T$. For the
end point, notice that the robot detected the event at $t=\tau$, and it will only be able to confirm
the event on times that are integer multiples of $\tau$. By the definition of $n$, $
n\tau<T\leq(n+1)\tau $. So when the robot observes the same event at
$(n+2)\tau$, it confirms that is has been there for more than $T$, since $
(n+2)\tau-\tau=(n+1)\tau\geq T $. Moreover, there can be cases for some $T$ and $t'$ when $t'+T<(n+1)\tau$, as shown in Figure~\ref{fig_stay_time}, but since the robot detected the event at $\tau$ and $(n+1)\tau-\tau=n\tau<T$, the robot does not know that the event has been active for more than $T$ and cannot confirm it.
\end{proof}

The events which become inactive in the interval (\ref{eq:conf_interval}) will be true events, but cannot be correctly classified by the robot as true. We will use this fact along with the exponential active times of the events in the following section to calculate the chances of correctly classifying a true event. 

\subsection{Probability of Correctly Classifying True Events}

The events which were detected at $t=\tau$ will be classified as true if they remain active until $t=(n+2)\tau$, as shown in Proposition~\ref{prop:false_negative_interval}. If the robot detects an event, then it knows that the event arrived in the interval between the last two visits to its vertex. By the property of stationary increments, the time scale can be shifted to say that the arrival time of the event is given by $t'\in(0,\tau]$. Using the consequence of Lemma~\ref{lem:arrival_dist}, the arrival time $t'$ is uniformly distributed over $(0,\tau]$. We write this distribution of $t'$ in $(0,\tau]$ as
\begin{equation}
\label{eq:arrival_dist}
\begin{aligned}
f(t')&=\frac{1}{\tau} & \text{for } t'\in(0,\tau].
\end{aligned}
\end{equation}

 So, if a robot knows than an event arrived within some two visits, it could have arrived at any time between those visits with equal probability. We will use this uniform density along with the interval(\ref{eq:conf_interval}) to find the probability of confirming true events.

\begin{proposition}[Probability of Successful Classification]
The probability of confirming a true event at vertex $v$ with arrival rate $\lambda$, departure rate $\mu$, and a robot with visit period $\tau$ is
\begin{equation}
\label{prob_confirm}
\prb{\texttt{\emph{confirm|v}}}=\frac{e^{-\mu[(n+2)\tau - T]}(e^{\mu\tau}-1)}{\mu\tau},
\end{equation}
where $n$ is defined in (\ref{eq:n_def}).
\end{proposition}

\begin{proof}
According to the confirmation avoiding interval given in (\ref{eq:conf_interval}), the events arriving at time $t'\in(0,\tau]$ and departing after $(n+2)\tau$ will be confirmed by the robot. Using the exponential staying time distribution ,we can find the probability of confirming a true event given that it arrived at time $t'\in(0,\tau]$.

\begin{equation}
\begin{split}
\label{eq:prob_overstay_t}
\prb{\texttt{confirm|v and t'}}&=\frac{\displaystyle\int_{(n+2)\tau}^{\infty}\mu e^{-\mu (t-t')}dt}{\displaystyle\int_{T}^{\infty}\mu e^{-\mu t}dt},\\
	   &=e^{-\mu[(n+2)\tau - T]}e^{\mu t'}
\end{split}
\end{equation}

The numerator in (\ref{eq:prob_overstay_t}) represents the events that will stay long enough to be confirmed, and the denominator represents all the events that are true.  Using the arrival time density in the interval $(0,\tau]$ from (\ref{eq:arrival_dist}), $f(t')=\frac{1}{\tau}$ where $t'\in(0,\tau]$, and un-conditioning the arrival time

\begin{equation*}
\label{eq:prob_caught_1}
\prb{\texttt{confirm|v}}=\int_{0}^{\tau}\prb{\texttt{confirm|v and t'}}f(t')dt',
\end{equation*}
 we get
\begin{equation*}
\begin{split}
\prb{\texttt{confirm|v}}&=\int_{0}^{\tau}\frac{e^{-\mu[(n+2)\tau - T]}e^{\mu t'}}{\tau}dt'\\
            &=\frac{e^{-\mu[(n+2)\tau - T]}(e^{\mu\tau}-1)}{\mu\tau}.
\end{split}
\end{equation*}

\end{proof}
  
The Event Detection and Confirmation problem seeks to maximize the number of true events that are confirmed. So, one would want to maximize the probability of confirmed true events given in (\ref{prob_confirm}). However, this expression is just for a single vertex on the path of the robot. We will extend it to the complete path, and then try to maximize the probability of confirming true events over the whole graph.

\section{A Single Robot Policy Based on the TSP}

The expression~\eqref{prob_confirm} can be used to find the probability of confirming true events over the whole graph. In this section, we will derive the expression for the probability, and then use it in a special case to recommend a policy based on the TSP tour of the graph. 

\subsection{Probability of Correct Classification for the Tour}
We start with the analysis of any patrolling policy with possibly different periodic visit times to vertices, and then specialize the equation for the case when the periodic visit times to the vertices are equal and the events' activity period is governed by the same process for all the vertices.  

Using equation (\ref{prob_confirm}), for a vertex $v$ with arrival and departure rates given by $\lambda_v$ and $\mu_v$ respectively, the robot visiting that vertex with a period $\tau_v$ will confirm true events on that vertex with a probability given by 

\begin{equation}
\label{eq:confirm|v}
\begin{split}
\prb{\texttt{confirm|v}}=\frac{e^{-\mu[(n_v+2)\tau_v - T]}(e^{\mu_v\tau_v}-1)}{\mu_v\tau_v},\\ \text{ where } 
  n_v=\begin{cases}
    \frac{T}{\tau_v}-1, & \text{if $T$ is a multiple of $\tau_v$}\\
    \big\lfloor\frac{T}{\tau_v}\big\rfloor, & \text{otherwise},
  \end{cases}
  \end{split}
\end{equation}

\begin{proposition}[Probability Expressions]
The probability of correctly classifying true events for the periodic tour is given by
\begin{equation}
\label{eq:diff_lmbda}
\prb{\texttt{confirm}}=\frac{\sum_{v}\prb{\texttt{\emph{confirm|v}}}\lambda_v}{\sum_{v}\lambda_v},
\end{equation} 
where $\prb{\texttt{\emph{confirm|v}}}$ is given in equation~(\ref{eq:confirm|v}).  Moreover, in the special case where $\tau_v=\tau$, and $\mu_v=\mu$, for all $v\in V$, then 
\begin{equation}
\label{prb_tsp}
\begin{split}
\prb{\texttt{confirm}}=\frac{e^{-\mu[(n+2)\tau - T]}(e^{\mu\tau}-1)}{\mu\tau},\\ \text{ where } 
  n=\begin{cases}
    \frac{T}{\tau}-1, & \text{if $T$ is a multiple of $\tau$}\\
    \big\lfloor\frac{T}{\tau}\big\rfloor, & \text{otherwise}.
  \end{cases}
  \end{split}
\end{equation}
\end{proposition}

\begin{proof}

We want to remove the condition of arrival being on a certain vertex $v$ from equation (\ref{eq:confirm|v}). We know that

\begin{equation*}
\prb{\texttt{confirm}}=\sum_{v}\prb{\texttt{confirm|v}}\prb{\texttt{arrival at v}}.
\end{equation*}

Since the arrivals of events at different vertices are independent processes, the probability of arrival of an event being on vertex $v$ is 
\begin{equation*}
\prb{\texttt{arrival at v}}=\frac{\lambda_v}{\sum_{v}\lambda_v}.
\end{equation*}

Therefore, 
\begin{equation*}
\prb{\texttt{confirm}}=\frac{\sum_{v}\prb{\texttt{confirm|v}}\lambda_v}{\sum_{v}\lambda_v}.
\end{equation*}

If we consider the special case when $\mu_v=\mu, \tau_v=\tau, \forall v\in V$, then $\prb{\texttt{confirm|v}}$ is same for all the vertices, and can be factored out of the summation, giving us equation (\ref{prb_tsp}).

\end{proof}

\begin{remark}[Dependence on $\lambda$]
Notice that the probability expression (\ref{prb_tsp}) is independent of $\lambda$. The reason can be understood by noting that in equation (\ref{eq:diff_lmbda}), the number of confirmed events depends on $\lambda$, but so does the number of total events.  As a result, the probability remain unaffected. \oprocend
\end{remark}

\subsection{Policy based on a TSP tour}
Expression~\eqref{prb_tsp} holds for the case when the period of visits for the robot is same for all vertices and the activity times of events at different vertices are identically distributed. A TSP tour (which is the shortest tour to visit all vertices in the graph) minimizes the time $\tau$ for a given speed of the robot.  

However, there are cases when decreasing the robot speed and thus increasing $\tau$ results in a higher probability in expression~\eqref{prb_tsp}. This is due to the discontinuity of $n$ in equation~\eqref{eq:n_def} and can be seen in Figure~\ref{fig:plot}. Intuitively it means that timing the visits such that $T$ is a multiple of $\tau$ tends to decrease the chances of missing the confirmation of true events. Thus, we need to check two possible robot speeds:  maximum speed, or moving at a speed such that $\tau$ is a multiple of $T$.  Based on this observation, we arrive at following single robot policy.  \\

\begin{figure}
  \centering
 \includegraphics[width=\linewidth]{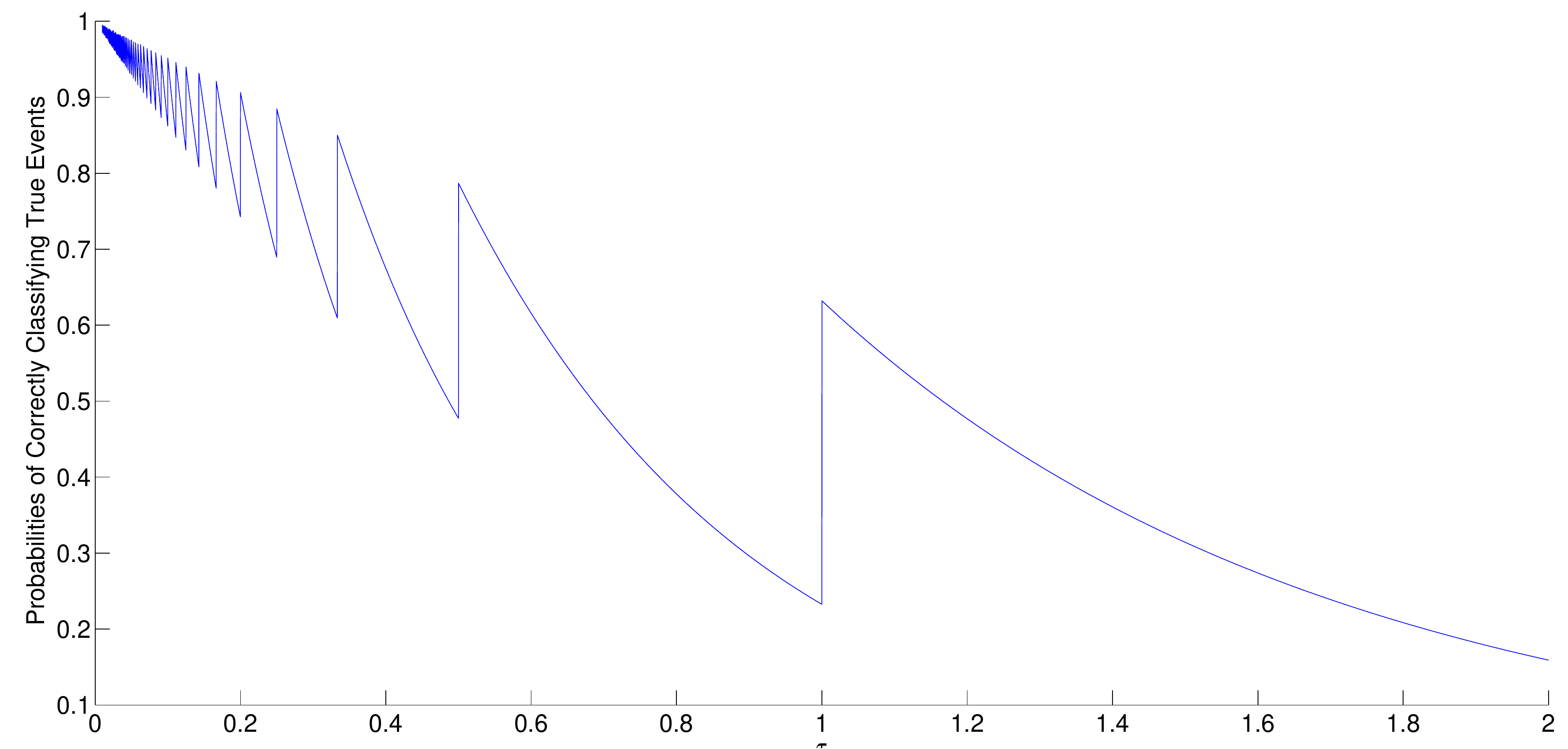}
  \caption{The probability of correctly classifying true events versus $\tau$, with $T=1$, $\mu=1$.}
  \label{fig:plot}
\end{figure}

\textbf{Policy for a Single Robot:}
\begin{enumerate}
\item Calculate the TSP tour of the graph, and find the minimum time $\tau_{\min}$ to complete that tour by the robot at its maximum speed.
\item Find $\tau_{\mathrm{peak}}\geq\tau_{\min}$ such that $T$ is a multiple of $\tau_{\mathrm{peak}}$ and there is no other $\tau$ between $\tau_{\min}$ and $\tau_{\mathrm{peak}}$ which is a multiple of $T$.
\item Calculate the probability of correctly classifying true events from equation (\ref{prb_tsp}) for $\tau=\tau_{\min}$ and $\tau=\tau_{\mathrm{peak}}$.
\item Choose the tour period which gives greater probability and adjust the speed of the robot to the optimal speed to match the chosen period. 
\end{enumerate}

\begin{remark}[Omitting Vertices from Tour]
Equation (\ref{eq:diff_lmbda}) suggests that missing some vertices on the tour can give a better probability. An instance of the problem with equal arrival rates can be easily constructed where missing a far away vertex from the tour will result in a much lower $\tau$ for the other vertices and hence increase the probability of correctly classifying true events over the whole graph. However, such policies raise the possibility that ``intelligent'' events would begin to choose this unvisited vertex more frequently, altering the arrival rates.  This becomes a problem in game theory, and thus we leave it for future work. \oprocend
\end{remark}

\section{Multiple Robots}
\label{sec:multiple_robots}


In this section we consider the case of multiple robots.  We assume that the communication graph between the robots is strongly connected, so that any two robots can communicate without significant delay. Thus, we can assume that the database containing all active events is shared among the robots.  This means that it is possible to have robot $i$ detect an event and robot $j$ confirm it.

\subsection{Specializing Robot Capabilities}

One possible solution in the multi-robot case is to utilize \emph{specialization} in which a robot performs exclusively detection, or exclusively confirmation.  
\begin{definition}[Specialized robot capability]
  We say that a robot is a detection (confirmation) robot if it is capable of performing only event detections (confirmations).
\end{definition}

First, it is easy to see that specialization cannot be optimal.  In specializing we eliminate the possibility of a confirmation robot detecting an event, even if it is the first robot to visit the vertex after the events' arrival.  Similarly, we eliminate the possibility of a detection robot confirming an event, even when it visits the event vertex more than $T$ time units after detection.  

However, there are cases where specialization may be required, for example, if the sensors needed for detection and confirmation differ.  The following simple lemma shows that when specializing, detection is the bottleneck.

\begin{lemma}[Specialization among robots]
\label{lem:conf_det}
Given $n_d$ detection robots, confirmation can be performed optimally using only $n_d$ confirmation robots.  
\end{lemma}
\begin{proof}
  Let the paths followed by the $n_d$ detection robots be $P_1,\ldots,P_{n_d}$.  We then create $n_d$ confirmation paths by placing a confirmation robot on each detection path, but with a time lag of exactly $T$ seconds.  

An event that is detected on a given path $P_i$ will be confirmed exactly $T$ time units later by the corresponding confirmation robot.  In other words, given the detection times, each confirmation will be performed optimally using the $n_d$ confirmation paths, with a time lag of $T$.
\end{proof}

The consequence of this result is that detection is the bottleneck when looking at specialized robots. Given detection patrolling paths, the corresponding optimal confirmation paths are defined.  Thus, in this case, one can use existing techniques to design patrolling paths for detection, and then use Lemma~\ref{lem:conf_det} for the confirmation paths.  In the next section we focus on the more complex case in which each robot can both detect and confirm.


\begin{figure*}[!t]
\normalsize
\newcounter{MYtempeqncnt}
\setcounter{MYtempeqncnt}{\value{equation}}
\setcounter{equation}{9}
\begin{equation}
\label{eq:case1}
\prb{\texttt{confirm}}=\begin{cases}
   						 \frac{1}{\mu\tau}(e^{-\mu((n+1)\tau-T)}(e^{\mu \tl}(1-e^{-\mu \tau})), & \tl\leq T-n\tau, \\
    					 \frac{1}{\mu\tau}(e^{-\mu((n+1)\tau-T)}(e^{\mu \tl}+e^{\mu(\tau-\tl)}-2), & T-n\tau<\tl\leq(n+1)\tau-T, \\
    					 \frac{1}{\mu\tau}(e^{-\mu((n+1)\tau+\tl-T)}(e^{\mu\tau}-1), & \tl>(n+1)\tau-T. \\
  					\end{cases}
\end{equation}

\begin{equation}
\label{eq:case2}
\prb{\texttt{confirm}}=\begin{cases}
   						 \frac{1}{\mu\tau}(e^{-\mu((n+1)\tau-T)}(e^{\mu \tl}(1-e^{-\mu \tau})), & \tl\leq(n+1)\tau-T, \\
    					 \frac{1}{\mu\tau}(e^{-\mu((n+1)\tau-T)}(2-e^{-\mu \tl}-e^{-\mu(\tau-\tl)}), & (n+1)\tau-T<\tl\leq T-n\tau, \\
    					 \frac{1}{\mu\tau}(e^{-\mu((n+1)\tau+\tl-T)}(e^{\mu\tau}-1), & \tl>T-n\tau. \\
  					\end{cases}
\end{equation}

\hrulefill
\setcounter{equation}{8}
\vspace*{4pt}
\end{figure*}


\subsection{Optimal Spacing Between Robots on a Common Path}

In this section we look at the case where each robot can both detect and confirm events.  We focus on the special case in which there are two robots moving along the same tour, with a period of $\tau > 0$.   We seek the optimal spacing of these two robots along the tour.  We discuss the extension to $m$ robots at the end of this section.

To this end, define the variable to optimize as $\tl$ which is the time lag between the first and second robot on the common tour. Since the robots travel the tour with period $\tau$, we have $\tl \in (0, \tau)$. Consider an event that arrives at a vertex at time $t' \geq 0$.  We can shift the time scale such that $t'\in [0, \tau)$. 

Then, let us consider the earliest possible time that this event can be detected and confirmed:  we call these times $\td$ and $\tc$, respectively, where $\tc \geq \td + T$ and $\td \geq t'$.

If these times are known, then the probability of confirming a true event, given that it arrives at time $t'$ is
\begin{align}
\label{eq:det_conf}
\prb{\texttt{confirm}|t'} &= \frac{\prb{\texttt{active} > \tc - t'}}{\prb{\texttt{active} > T}} \nonumber \\
&= \frac{e^{-\mu(\tc - t')}}{e^{-\mu T}} \nonumber \\
&= e^{\mu t'} e^{-\mu(\tc - T)},
\end{align}
where we have used the fact that an event's active time is distributed according to an exponential random variable with parameter $\mu$. Also, recall that $n$ is defined in equation~\eqref{eq:n_def}. Now, we can calculate $\td$ and $\tc$ as a function of $\tl$ using the following two cases, each containing two sub-cases. \\

\noindent \emph{Case 1:}  If $t'\in (0,\tl]$ then the event will be detected at $\td = \tl$.  There are two further sub-cases.
\begin{enumerate}
\item If $\tl+T \leq (n+1)\tau$ then the earliest time that the event can be confirmed is  $\tc = (n+1)\tau$. 
\item If $\tl+T > (n+1)\tau$, then the earliest time that the event can be confirmed is  $\tc = (n+1)\tau+\tl$.
\end{enumerate}

\noindent\emph{Case 2:}  If $t' \in (\tl,\tau]$ then the event will be detected at $\td = \tau$.  Again, there are two sub-cases:
\begin{enumerate}
\item  If $\tau + T \leq (n+1)\tau + \tl$ i.e., $\tl \geq T - n\tau$, then the earliest time that the event can be confirmed is $\tc = (n+1)\tau + \tl$.

\item If $\tl < T - n\tau$ , then the earliest time that the event can be confirmed is $\tc = (n+2)\tau$.
\end{enumerate}

Based on the four cases and equation~\eqref{eq:det_conf}, we can compute the probability of detection as a function of $\tl$ as 
\begin{equation*}
\prb{\texttt{confirm}}=\int_{0}^{\tau}\prb{\texttt{confirm}|t'}f(t')dt',
\end{equation*}  
where $f(t')$ is the uniform distribution from equation (\ref{eq:arrival_dist}). 

When evaluating this integral, there are two more cases, depending on whether or not  $T-n\tau  \leq (n+1)\tau-T$: 
\begin{itemize}
\item If $T-n\tau  \leq (n+1)\tau-T$ we get equation~\eqref{eq:case1}.
\item If $T-n\tau > (n+1)\tau-T$, we get equation~\eqref{eq:case2}.
\end{itemize}
Now, from these expressions we can optimize $\tl$.   Notice that in the case when $\tl = \tau/2$, the expressions in~\eqref{eq:case1} and~\eqref{eq:case2} both simplify to equation~\eqref{prb_tsp} with $\tau$ replaced by~$\tau/2$.    

\begin{proposition}[Optimal value of $\tl$]
The equations~\eqref{eq:case1} and (\ref{eq:case2}) achieve their global maxima at one (or more) of the following points: i) $\tl=\frac{\tau}{2}$;  ii) $\tl=T-n\tau$; or iii) $\tl=(n+2)\tau-T$.
\end{proposition}

\begin{proof}
The equations~\eqref{eq:case1} and (\ref{eq:case2}) are piecewise continuous and have discontinuities at points $\tl=T-n\tau$ and $\tl=(n+2)\tau-T$. The continuous pieces defined on the first and third interval of the equations are strictly monotone and achieve their maximum values at the discontinuities. The third continuous part has an extremum at $\tl=\frac{\tau}{2}$. Thus its maximum lies either at $\frac{\tau}{2}$ or at one of the discontinuities.  So, for one of these values of $\tl$, the probability will be maximized.
\end{proof}

These values of $\tl$ will optimize the probability for a given value of $\tau$. However, as we observed in the single robot case, decreasing the speed of the robot to increase $\tau$ to a divisor of $2T$ can result in a higher probability.  Based on this, we arrive at the following policy for two robots: \\

  \textbf{Policy for Optimizing the Spacing of Two Robots:}

\begin{enumerate}
\item Evaluate the expression~\eqref{eq:case1} or (\ref{eq:case2}) depending on whether $T-n\tau\leq(n+1)\tau-T$ or not, at the points $\tl =\tau/2$, $\tl=T-n\tau$ and $\tl=(n+2)\tau-T$. Choose the value of $\tl$ that gives the maximum probability. Call it $t_{\mathrm{lag_1}}$.
\item Decrease $\tau$ to the nearest divisor of $2T$, call it $\tau_n$ and evaluate equation (\ref{eq:case2}) at $\tl=\tau_n/2$ for $\tau=\tau_n$.
\item If $\tl=\tau_n/2$ gives the higher probability, choose $\tau=\tau_n$ and $\tl=\tau_n/2$. Otherwise, choose $\tl=t_{\mathrm{lag_1}}$.
\end{enumerate}

The following remark discusses the extension to $m$ robots.

\begin{remark}[Generalizing to $m$ robots]  In the case that there are $m$ robots, there are $m-1$ variables  $\tl$ to optimize.  The number of cases to consider becomes too large to complete the same analysis.  However, based on the observations made for two robots, the following can be said for the $n$-robot case: 
\begin{enumerate}
\item If $\tau/m$ is a multiple of $T$, then equally space the robots on the path. \item Otherwise, if $\tau<mT$, decrease $\tau$ to the nearest divisor of $mT$ and using this new period, equally space the robots. 
\item If $\tau>mT$, choose the spacing such that the robots follow each other by a time lag of $T$.
\end{enumerate}
This policy follows from the observation that in the two robot case this procedure often yields the optimal $\tl$.  However, it is not, in general guaranteed to find the optimal spacing.   \oprocend
\end{remark}

\section{Application Example}

Let us consider the parking lot shown in Figure~\ref{fig:parking_example} and apply the patrolling policies to that parking lot. The parking lot is situated at a market place and for the purposes of simulation, we assume the expected staying time of vehicles to be around 75 minutes and the allowed parking time to be two hours. This gives $\mu=\frac{1}{75}$ and $T=120$. Moreover, the length of the patrolling path shown in Figure~\ref{fig:parking_example} is calculated to be approximately 870 meters. 

Figure~\ref{fig:example1} shows the plot of probabilities of ticketing overstaying vehicles versus patrolling period of the robot in the cases of i) a single robot; ii) two robots with a spacing of $\tau/2$; and iii) two robots with optimized spacing.

We assume the robot has a maximum speed of $1$ m/s, which gives a minimum tour period of $14.5$ minutes. The probability of ticketing for a tour with this period comes out to be $0.7905$ from equation~\eqref{prb_tsp}. However, if we increase the period to 15 minutes by decreasing the speed of robot to $0.967$ m/s, using equation~\eqref{prb_tsp} the probability increases to $0.9063$  --- an increase of $14.6\%$. 

In case of two robots, the probability of ticketing an overstaying vehicle using a period of 14.5 minutes and the robots equally spaced will be 0.9128 from equation~\eqref{eq:case2} .  But, if one robot follows the other with a time lag of $T-n\tau=4$ minutes or $(n+1)\tau-T=10.5$ minutes (since the robots are on a cycle, these configurations are equivalent), then the probability increases to 0.9221 from equation~\eqref{eq:case1}.
However, there is still room for improvement. If both the robots decrease their speed to make their period a multiple of $T$ and then follow each other with a lag of $\tau/2=7.5$ minutes, the probability will be 0.9515 which can be calculated using equation~\eqref{eq:case2} or equation~\eqref{prb_tsp} with $\tau$ replaced by $\tau/2$. 

This example shows that decreasing the speed can be helpful in terms of increasing the probability of confirming a true event. However, it may not always be true. Looking at Figure~\ref{fig:example1} that most of the times a lag of $\tau/2$ would result in better probability, and even if it does not, the optimal lag seems to provide very little advantage. However, Figure~\ref{fig:example2} shows that in the cases when $\tau>2T$ in a two robot scenario, optimal lag provides much better results. Such cases can arise when robots have to monitor very large environments.

\begin{figure}
  \centering
 \includegraphics[width=\linewidth]{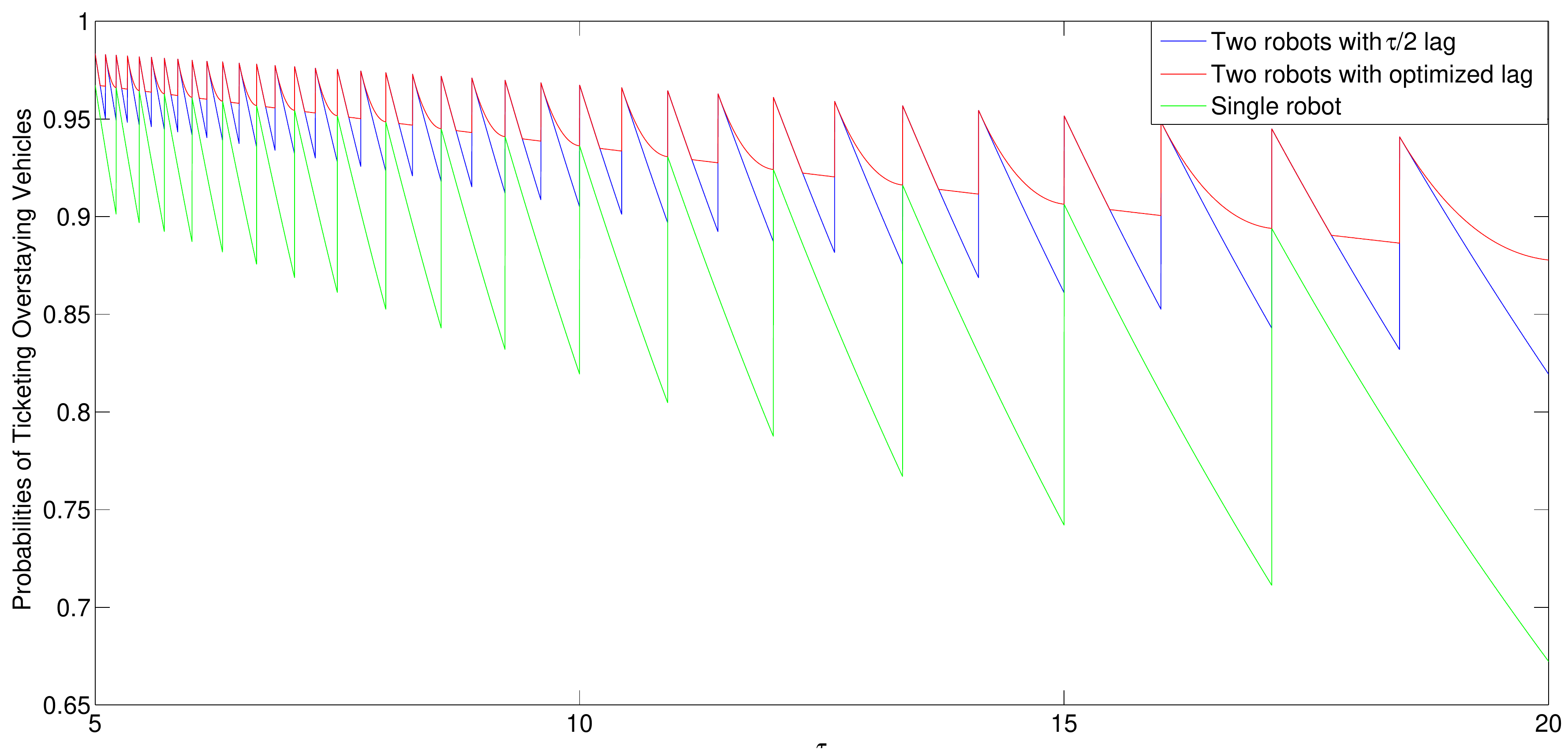}
  \caption{The probability of ticketing overstaying vehicles as a function of $\tau$ for 1) a single robot, 2) two robots with lag of $\tau/2$, and 3) two robots with optimized lag.  Here $T=120$, and $\mu=1/75$.}
  \label{fig:example1}
\end{figure}

\begin{figure}
  \centering
 \includegraphics[width=\linewidth]{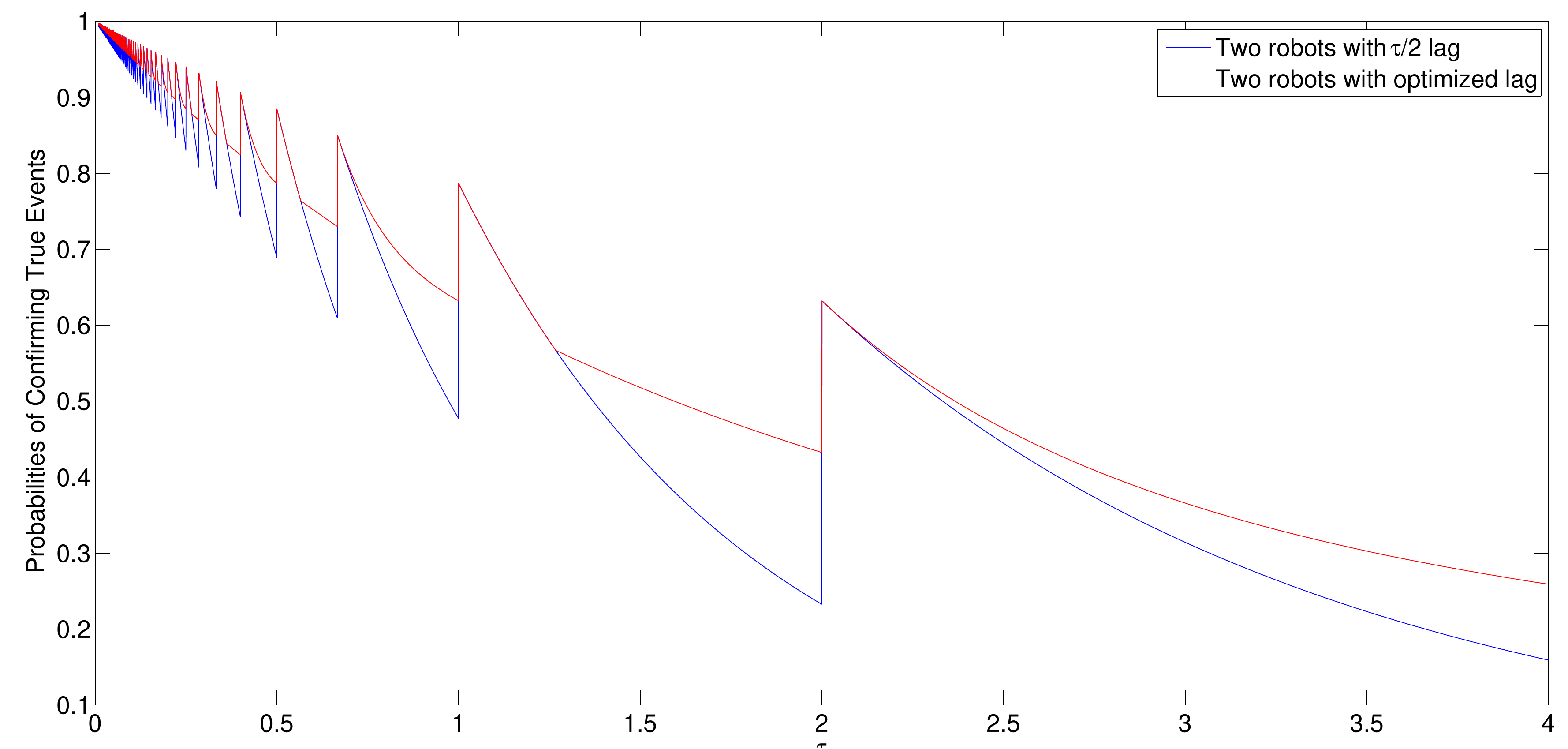}
  \caption{Comparison of probabilities of confirming events for two robots with different lags versus $\tau$.  Parameter values are $T=1$, $\mu=1$.}
  \label{fig:example2}
\end{figure}

\section{Conclusions and Future Work}
\label{sec:conclusions}

In this paper we considered a robot patrolling problem called the event detection and monitoring problem.  We showed that the off-line version of this problem is NP-hard.  We considered a simple patrolling policy based on the traveling salesman tour, and characterized the probability of confirming a true event.  We gave some initial insights into the multiple robot problem.  In particular we showed that the robots all follow the same path, the optimal spacing can be nonuniform, depending the time $T$ and the period $\tau$.  

For future work we would like to solve the dual problem, which is to minimize the number of robots needed to achieve a desired confirmation probability.  We would also like to compare the performance of a TSP tour with that of a min-max latency tour~\cite{Alamdari2014Persistent}, where visit frequency can be proportional to the event arrival rate at a given vertex.  We would also like to study the problem from a game theoretic perspective, where the events distribution may change as a function of the patrolling policy.  For example, where people begin to alter their parking behaviors based on the patrolling path.  In this case we need to look at randomizing the path, in order to decrease its predictability.  Randomized policies have been considered in perimeter patrolling problems~\cite{Agmon2008Multi}, and thus will form a solid basis on which to build.


\bibliographystyle{IEEEtran}
\bibliography{references}

\begin{thebibliography}{10}
\providecommand{\url}[1]{#1}
\csname url@samestyle\endcsname
\providecommand{\newblock}{\relax}
\providecommand{\bibinfo}[2]{#2}
\providecommand{\BIBentrySTDinterwordspacing}{\spaceskip=0pt\relax}
\providecommand{\BIBentryALTinterwordstretchfactor}{4}
\providecommand{\BIBentryALTinterwordspacing}{\spaceskip=\fontdimen2\font plus
\BIBentryALTinterwordstretchfactor\fontdimen3\font minus
  \fontdimen4\font\relax}
\providecommand{\BIBforeignlanguage}[2]{{%
\expandafter\ifx\csname l@#1\endcsname\relax
\typeout{** WARNING: IEEEtran.bst: No hyphenation pattern has been}%
\typeout{** loaded for the language `#1'. Using the pattern for}%
\typeout{** the default language instead.}%
\else
\language=\csname l@#1\endcsname
\fi
#2}}
\providecommand{\BIBdecl}{\relax}
\BIBdecl

\bibitem{chevaleyre2004theoretical}
Y.~Chevaleyre, ``Theoretical analysis of the multi-agent patrolling problem,''
  in \emph{Intelligent Agent Technology, IEEE/WIC/ACM International Conf. on},
  2004, pp. 302--308.

\bibitem{Agmon2008Multi}
N.~Agmon, S.~Kraus, and G.~Kaminka, ``Multi-robot perimeter patrol in
  adversarial settings,'' in \emph{Robotics and Automation. IEEE Int. Conf.
  on}, May 2008, pp. 2339--2345.

\bibitem{baseggio2010distributed}
M.~Baseggio, A.~Cenedese, P.~Merlo, M.~Pozzi, and L.~Schenato, ``Distributed
  perimeter patrolling and tracking for camera networks,'' in \emph{Decision
  and Control, IEEE Conf. on}.\hskip 1em plus 0.5em minus 0.4em\relax IEEE,
  2010, pp. 2093--2098.

\bibitem{BK-JV:07}
B.~Korte and J.~Vygen, \emph{Combinatorial Optimization: Theory and
  Algorithms}, 4th~ed., ser. Algorithmics and Combinatorics.\hskip 1em plus
  0.5em minus 0.4em\relax Springer, 2007, vol.~21.

\bibitem{Pasqualetti2012Cooperative}
F.~Pasqualetti, A.~Franchi, and F.~Bullo, ``On cooperative patrolling: Optimal
  trajectories, complexity analysis, and approximation algorithms,''
  \emph{Robotics, IEEE Transactions on}, vol.~28, no.~3, pp. 592--606, June
  2012.

\bibitem{Alamdari2014Persistent}
S.~Alamdari, E.~Fata, and S.~L. Smith, ``Persistent monitoring in discrete
  environments: Minimizing the maximum weighted latency between observations,''
  \emph{International Journal of Robotics Research}, vol.~33, pp. 138--154,
  2014.

\bibitem{savelsbergh1985local}
M.~W. Savelsbergh, ``Local search in routing problems with time windows,''
  \emph{Annals of Operations research}, vol.~4, no.~1, pp. 285--305, 1985.

\bibitem{GL:09}
G.~Laporte, ``Fifty years of vehicle routing,'' \emph{Transportation Science},
  vol.~43, no.~4, pp. 408--416, 2009.

\bibitem{ropke2006adaptive}
S.~Ropke and D.~Pisinger, ``An adaptive large neighborhood search heuristic for
  the pickup and delivery problem with time windows,'' \emph{Transportation
  science}, vol.~40, no.~4, pp. 455--472, 2006.

\bibitem{Bullo2011Dynamic}
F.~Bullo, E.~Frazzoli, M.~Pavone, K.~Savla, and S.~Smith, ``Dynamic vehicle
  routing for robotic systems,'' \emph{Proceedings of the IEEE}, vol.~99,
  no.~9, pp. 1482--1504, Sept 2011.

\bibitem{MP-NB-EF-VI:08}
M.~Pavone, N.~Bisnik, E.~Frazzoli, and V.~Isler, ``A stochastic and dynamic
  vehicle routing problem with time windows and customer impatience,''
  \emph{ACM/Springer Journal of Mobile Networks and Applications}, vol.~14,
  no.~3, pp. 350--364, 2009.

\bibitem{GG-DR:01}
G.~Grimmett and D.~Stirzaker, \emph{Probability and Random Processes}.\hskip
  1em plus 0.5em minus 0.4em\relax Oxford University Press, 2001.

\bibitem{ross1996stochastic}
S.~M. Ross, \emph{Stochastic processes}.\hskip 1em plus 0.5em minus 0.4em\relax
  John Wiley \& Sons New York, 1996, vol.~2.

\bibitem{Kleinrock1975Queueing}
L.~Kleinrock, \emph{Queueing Systems. Volume I: Theory}.\hskip 1em plus 0.5em
  minus 0.4em\relax John Wiley, 1975.

\end{thebibliography}

\end{document}